\pdfoutput=1
\documentclass[twocolumn, journal]{IEEEtran}
\usepackage[T1]{fontenc}
\usepackage{amsmath,amssymb,amsfonts}
\usepackage{amsthm}
\usepackage{algorithm}
\usepackage{algpseudocode}
\usepackage{graphicx}
\usepackage{booktabs}
\usepackage{array}
\usepackage{multirow}
\PassOptionsToPackage{dvipsnames}{xcolor}
\usepackage{xcolor}
\newtheorem{remark}{\bf Remark}
\usepackage{url}
\usepackage{cite}
\usepackage{bm}
\usepackage{dsfont}
\usepackage[colorlinks=true,allcolors=blue]{hyperref}

\newtheorem{proposition}{Proposition}

\begin{document}

\title{Communication-Efficient Digital-Twin Coordination for Heterogeneous LLM Embodied Agents over Computing Power Networks}

\author{Nuocheng~Yang, \emph{Student Member, IEEE}, Sihua~Wang, Zihan~Chen, \emph{Member, IEEE}, \\ Tony Q. S. Quek,~\IEEEmembership{Fellow,~IEEE}, and~Changchuan~Yin, \emph{Senior Member, IEEE}
\thanks{N. Yang, S. Wang, and C. Yin are with the Beijing Laboratory of Advanced Information Network, and the Beijing Key Laboratory of Network System Architecture and Convergence, Beijing University of Posts and Telecommunications, Beijing 100876, China (emails: \{yangnuocheng, sihuawang, ccyin\}@bupt.edu.cn).}
\thanks{Z. Chen and T. Q.~S. Quek are with the Information Systems Technology and Design Pillar, Singapore University of Technology and Design, 487372, Singapore (emails: zihan\_chen@mymail.sutd.edu.sg, tonyquek@sutd.edu.sg).}
}

\maketitle

\begin{abstract}
Embodied agent teams powered by heterogeneous large language models (LLMs) are being widely deployed in physical artificial intelligence such as smart factories, warehouses, and service robotics.
To enable collaboration among such an agent team, efficient coordination mechanisms that operate reliably under limited network resources are required.
However, existing heterogeneous LLM-agent coordination frameworks that rely on multi-round natural-language-based conversations introduce three coupled challenges.
First, inter-agent dialogue incurs communication overhead that grows rapidly with team size.
Second, the quality of coordination is constrained by the heterogeneous capabilities of the agent team's LLMs.
Third, agents may suffer from action delays due to iterative negotiation.
To address these challenges, we propose LDT-Coord, a networked coordination framework built upon a lightweight digital twin (DT).
Specifically, each agent independently selects its intended action and reports both the action decision and a structured temporal constraint over shared resources to the DT server, thereby decoupling coordination performance from natural-language reasoning ability.
Then, DT executes a training-free, rule-based orchestrator algorithm to resolve cross-agent conflicts and returns coordination instructions to prevent such conflicts.
To further reduce communication overhead, we formulate agent reporting control as a constrained partially observable Markov decision process (C-POMDP) and solve it with the PPO-Lagrangian algorithm.
Simulation results show that LDT-Coord achieves a task success rate comparable to conventional coordination methods while reducing communication overhead by more than $70\times$ and maintaining robustness under LLM heterogeneity.

\end{abstract}


\section{Introduction}

Large language models (LLMs) have shown strong capability in understanding, reasoning, and generation \cite{Zhao2026LLMSurvey,Qin2026LLMNLP,Xu2025LargeReasoning}, which drives an embodied agents framework that can perceive, reflect, and act in a physical artificial intelligence world \cite{Wang2025EmbodiedIntelligence,Chowa2026LLMAgentsReview,Sun2025LLMMultiAgent}.
Compared with the conventional centralized framework, where a single server controls each agent, the embodied agents framework can enable distributed intelligence, thereby improving responsiveness, scalability, and robustness in dynamic environments.
As these agents scale from single models to teams, members may run LLMs with heterogeneous capabilities to fit each agent's task and resource budget \cite{Ahn2022SayCan,Driess2023PaLME,Brohan2023RT2,Liang2023CodeAsPolicies}.
To enable efficient coordination of such a heterogeneous agent team, several works focus on facilitating information exchange through multiple rounds of natural language (NL) dialogue \cite{Mandi2023RoCo,Wu2023AutoGen,Hong2023MetaGPT}. 
However, building coordination on top of mutual understanding of each other's NL output exposes three issues. 
First, at the communication level, the payload of multi-round NL negotiation grows rapidly with team size and number of rounds, which introduces an expensive communication cost. 
Second, at the heterogeneity level, the quality of NL dialogue suffers from a short-board effect, which may be significantly affected by the participant with the weakest LLM. 
Finally, at the mechanism level, agents are required to dialogue before acting, thereby introducing cooperation latency.

Existing work addresses these issues along three technical routes. 
To address the communication bottleneck, the authors in \cite{Salemi2025Blackboard,Zhang2026MiTa,Wang2025KARMA} proposed a shared-state and blackboard method that transforms mesh dialogue into star-shaped reads and writes over a shared memory.
However, the shared medium remains NL text, so vague descriptions written by agents with the weakest LLMs may continue to pollute the shared state, thereby reducing collaborative efficiency and performance. 
To address the heterogeneity bottleneck, the authors in  \cite{Parada2025IntNet,Su2025CommEfficientMARL,Charalambous2026GoalOriented,Hu2024CommFormer} studied multi-agent reinforcement learning communication routes that learn compact message vectors to bypass NL.
For the mechanism bottleneck, the authors in \cite{Liu2025COHERENT,Zhang2025LaMMAP,Li2025AdaptiveLLMOrchestration} introduce a centralized LLM planner that generates joint actions for all agents.
Although this architecture can avoid action conflicts through a single strong coordinator, it shifts distributed intelligence toward centralized coordination, which suffers from scalability difficulties as the number of agents grows \cite{Chen2024Scalable}.

Meanwhile, the digital twin (DT) serves as a high-fidelity digital mirror of a physical entity in the information space \cite{Zhou2025DTUAVTopology,Tran2025NetworkDT,Li2025DTSync}.
Through continuous state synchronization and bidirectional interaction, it has become a unified carrier for sensing, prediction, and control in networked systems \cite{Grieves2017Mitigating,Tao2019Industry}. In multi-agent settings, it has been further used as a centralized mirror and reconstruction tool for the team state, aggregating scattered local observations and action into a consistent global view \cite{Chen2025Goal,Zhang2021Scheduling}. However, existing DT-based schemes \cite{Chen2025Goal,Zhang2021Scheduling} still adopt a central-server architecture in which the DT actively generates a plan for each agent similar to \cite{Liu2025COHERENT,Zhang2025LaMMAP,Li2025AdaptiveLLMOrchestration}. 
This introduces uncontrollable latency and computational overhead, as DT must re-understand the physical world from scratch, even though agents have already formed a local understanding of their surroundings. 

To fill this gap, this paper proposes Lightweight Digital-Twin Coordination (LDT-Coord), which employs a lightweight DT as a coordination middleware to enable efficient collaboration among embodied agent teams driven by heterogeneous LLMs.
Specifically, each embodied agent autonomously perceives, reflects, acts, and reports its chosen action, together with structured temporal constraints on shared resources, to the DT for coordination.
Thus, the lightweight DT does not need to understand the environment from scratch as the traditional method, thereby saving computational and communication resources. 
Instead, it uses a unified lightweight orchestrator that operates over a shared set of structured coordination primitives, detects and avoids potential conflicts in the cooperation process by rules, and returns efficient coordination. 
The main contributions of this paper are summarized as follows.
\begin{itemize}
\item \textbf{Lightweight-DT coordination framework with structured communication primitives.} We use a lightweight DT as the coordination middleware for heterogeneous LLM agents and define structured primitives, namely the action tuple, the typed constraint declaration, and a short downlink instruction, which turn coordination from mutually understanding language into reporting structured constraints and decouple coordination quality from the language ability of the reporter.

\item \textbf{Training-free unified rule-based orchestrator for atomic-task conflict avoidance.} We formalize the mutual-exclusion, synchronization, and dependency conflict types into typed coordination rules and design a unified orchestrator applied until convergence to find the maximal consistent executable set without any training.

\item \textbf{Learned communication-selection layer for latency-constrained reporting.} We model the decision of which agents report at each step as a constrained partially observable Markov decision process (C-POMDP) under a per-step latency constraint and solve it with PPO-Lagrangian, which compresses state-reporting communication substantially while keeping coordination quality nearly lossless.
\end{itemize}

Experiments show that, compared with the traditional NL-dialogue coordination strategy among agents, LDT-Coord attains a comparable success rate while reducing communication by more than 70$\times$ on the Confined-Space Sorting task, and it stays robust across heterogeneous team configurations of different scales.

\section{Related Work}

\subsection{Natural-Language Dialogue-Based Multi-Agent Coordination}

Recently, a number of works that employ the NL dialogue route have been studied to enable LLM agents to cooperate.
The authors in \cite{Mandi2023RoCo} proposed a dialogue-based scheme that lets multiple robotic arms discuss task assignment and collision-free paths through multi-round NL exchange.
The authors in \cite{Wu2023AutoGen} introduced a conversable-agent framework that orchestrates multi-agent dialogue flows through a group-chat manager.
The authors in \cite{Hong2023MetaGPT} proposed a procedure-driven scheme that encodes standardized operating procedures into the communication among agents to reduce cascading errors.
The authors in \cite{Salemi2025Blackboard} systematized a shared-memory architecture so that agents coordinate by reading and writing a blackboard rather than talking pairwise.
The authors in \cite{Li2023CAMEL} drove two-agent cooperation through a role-playing dialogue mechanism.
The authors in \cite{Qian2023ChatDev} organized software development as a multi-role dialogue pipeline.
These frameworks coordinate from the perspective of agents understanding one another through language outputs.
Under heterogeneous deployment, the communication quality is dictated by the weakest party. Moreover, intent is not shared explicitly, so each agent must infer the plans of others from their text.

To improve the dialogue itself, another line introduces reflection, debate, and interleaved reasoning.
The authors in \cite{Shinn2023Reflexion} proposed a self-reflective scheme that lets an agent iteratively improve its decisions through verbalized self-feedback.
The authors in \cite{Du2023MultiAgentDebate} proposed a multi-party argument scheme that improves factuality and reasoning through debate among agents.
The authors in \cite{Yao2022ReAct} proposed a reasoning-acting scheme that interleaves reasoning with action and environment feedback to stabilize embodied decisions.
The authors in \cite{Park2023GenerativeAgents} proposed a memory-driven scheme that simulates believable behavior through NL memory and reflection.
The authors in \cite{Wang2023Voyager} introduced a curriculum scheme for continual skill acquisition in an open world.
The authors in \cite{Liu2023DyLAN} proposed a dynamic agent network that forms teams by selecting agents according to contribution.
These methods raise dialogue quality but spend many time steps on negotiation that yields no motion, while weak agents keep polluting the shared context.
More fundamentally, these frameworks generally adopt a mesh dialogue topology. As a result, the negotiation links, payload, and latency grow rapidly with team size, which is especially acute in heterogeneous, link-constrained deployment.

\subsection{Centralized and Hierarchical Multi-Robot Task Coordination}

To circumvent the inherent defect of mesh dialogue, a number of works turn to centralized and hierarchical coordination. A central planner then generates actions for all agents in a unified fashion.
The authors of \cite{Liu2025COHERENT} proposed a centralized orchestration scheme that decomposes and dispatches subtasks across a heterogeneous team of agents.
The authors in \cite{Liu2023LLMP} introduced a hybrid scheme that couples an LLM with a classical symbolic planner to obtain feasible long-horizon plans.
The authors in \cite{Lin2023Text2Motion} proposed a language-to-motion scheme that generates executable motion sequences from NL instructions.
These works remove conflicts from the perspective of centralized planning, but their coordination quality depends entirely on the capability of the central model. They also require uploading all observations to the center, so communication and computation scale poorly with team size.

Classical multi-robot task allocation instead characterizes coordination from the perspective of combinatorial optimization and scheduling.
The authors in \cite{Gerkey2004Taxonomy} introduced a formal taxonomy and analysis of multi-robot task allocation.
The authors in \cite{Wang2025MATP} proposed a human-robot collaborative assembly planning method based on LLM agents.
These methods assume homogeneous capability and sufficient communication by default and do not target training-free coordination of heterogeneous LLM teams under constrained links.
Therefore, the centralized route trades a single strong center for coordination quality, which runs counter to the decentralized autonomy and heterogeneity-robust objectives pursued in this paper.

\subsection{Multi-Agent Communication and Learned Coordination}

To circumvent the bloated negotiation payload and the message meaning that fluctuates with the sender model, multi-agent reinforcement learning has widely explored the learned communication route.
The authors in \cite{Sukhbaatar2016CommNet} proposed a differentiable communication scheme that learns continuous messages by back-propagating across a shared channel among agents.
The authors in \cite{rashid2018qmix} proposed a value-decomposition scheme that supports centralized training with decentralized execution through monotonic mixing.
The authors in \cite{Ding2024MultiLevelComm} introduced a multi-level communication scheme that lets agents exchange messages in a determined order to coordinate their decisions.
These works improve coordination by end-to-end learning of compact protocols. However, their messages and policies must be jointly trained and lack interpretability, so they are difficult to plug into a training-free heterogeneous LLM team.

The theoretical frameworks of decentralized partially observable coordination further characterize the difficulty of this problem.
The authors in \cite{oliehoek2016decpomdpbook} introduced a systematic formalization that reveals the intrinsic complexity of solving joint policies under partial observability.
The authors in \cite{duan2024gacg} proposed a group-aware coordination graph that characterizes the coordination structure among agents.
These methods rigorously characterize coordination among multiple agents from a re-learning perspective, but they still presuppose a joint policy obtained through retraining. This introduces additional training overhead and is difficult to combine quickly with existing teams.

\subsection{Distributed Consistency and Consensus Control}

Concurrency control and consensus in distributed systems provide classical mechanisms for multiple nodes to stay consistent without a central arbiter.
The authors in \cite{Kung1981Optimistic} proposed an optimistic concurrency control strategy that lets concurrent transactions execute first and validate conflicts at commit time with rollback as needed.
The authors in \cite{Lamport1978LogicalClocks} introduced a logical-clock scheme that characterizes the causal precedence among distributed events through a happens-before partial order.
The authors in \cite{Lamport1998Paxos} proposed a consensus protocol for strongly consistent replication across unreliable nodes.
The authors in \cite{Shapiro2011CRDT} proposed conflict-free replicated data types that converge through algebraic structure without explicit coordination.
These mechanisms guarantee correctness from the perspective of reaching agreement on shared state, but their nodes are homogeneous and faithful. Moreover, the objects they coordinate are data reads and writes or replica state, so they neither carry the action intent produced by heterogeneous LLMs nor handle physical-reachability constraints.

Multi-agent consensus from a control-theoretic perspective is closer to networked coordination.
The authors in \cite{ru2025secureconsensus} proposed a secure consensus protocol that gives robust convergence under mixed attacks.
The authors in \cite{liu2025doublelayerconsensus} proposed a double-layer finite-time consensus scheme for heterogeneous dynamics.
The authors in \cite{Lu2021Blockchain} proposed a blockchain-based scheme for trustworthy collaboration in digital-twin edge networks.
These works characterize the convergence of continuous physical quantities to a common value under disturbance and assume homogeneous agents by default. This focus is orthogonal to the discrete action conflicts of heterogeneous LLMs over discrete shared resources targeted here.

\subsection{Digital Twin for Networked Coordination}

The DT provides a digital replica of the physical world for networked systems \cite{Grieves2017Mitigating,Tao2019Industry}.
A recent line of work treats the DT as a coordination layer for multi-agent and multi-robot teams.
The authors in \cite{Xu2023digital} proposed a digital-twin scheme that trains a multi-agent scheduler to orchestrate heterogeneous edge-end devices and finish deadline-constrained jobs.
The authors in \cite{Tang2023digital} proposed a digital-twin scheme that acts as a coordination layer to assign and reassign tasks across a multi-UAV fleet.
The authors in \cite{Zhang2021Scheduling} introduced a digital-twin scheme for dynamic manufacturing-workshop scheduling based on multi-agent deep reinforcement learning.
These works let the DT allocate tasks or schedule jobs across a team. However, the twin coordinates assignment decisions rather than the discrete action intent that heterogeneous agents have already expressed in structured form.

Another line turns the DT into an active control layer that directly drives online physical-world decisions.
The authors in \cite{Zhou2025DTUAVTopology} proposed an end-edge collaborative framework that uses a spatio-temporal twin to online command a UAV team to adjust its neighbor topology for multi-target tracking.
The authors in \cite{Hevesli2024DTTaskOffloading} introduced a digital-twin edge air-ground scheme that jointly schedules UAV trajectories, device association, and task offloading through the twin's real-time prediction.
The authors in \cite{Wang2024SmartMobilityDT} proposed a smart-mobility platform that actively coordinates the routes of connected automated vehicles through cloud twin state.
The authors in \cite{Abishu2024DTResourceAllocation} introduced an internet-of-vehicles twin network that drives adaptive resource allocation and task offloading.
These works confirm the feasibility of a twin actively driving physical-world decisions. However, the objects they coordinate are continuous network resources, trajectories, or offloading ratios, and their mechanism is centralized optimization of a single global objective. 
As a result, they neither target the discrete action conflicts of heterogeneous LLM decision-makers nor arbitrate the action intent that each agent has already expressed in structured form.
At a wider scale, the computing power network (CPN) paradigm casts coordination as the joint orchestration of heterogeneous cloud-edge-end computing resources together with the network that connects them \cite{Tang2022diten}.
The authors in \cite{Tang2023siats} proposed an intent-aware scheduling framework that matches application demands to heterogeneous computing nodes and network paths. 
The digital-twin schedulers discussed above \cite{Xu2023digital,Tang2023digital} already act as computing-power-aware orchestration planes at the edge-end tier. 
However, these works orchestrate divisible computing jobs and continuous resources under a single global objective, rather than the discrete action intent of heterogeneous LLM decision-makers. 
Despite this progress, no existing route can actively coordinate the discrete action intent of heterogeneous LLM agents under training-free, low-communication, and heterogeneity-robust conditions, which is precisely the gap targeted by the system model and problem formulated in the next section.

Table~\ref{tab:related_work} organizes the comparison around the novelty of this paper and contrasts the key differences between existing routes and LDT-Coord.

\begin{table*}[t]
\centering
\caption{Comparison of Related Work}
\label{tab:related_work}
\renewcommand{\arraystretch}{1.3}
\small
\setlength{\tabcolsep}{3pt}
\begin{tabular}{@{}p{3cm}p{5cm}p{3cm}p{1.8cm}p{3.0cm}p{1.1cm}@{}}
\toprule
Category & Representative Works & Comm.\ Topology & Shared Intent & Heterogeneity-Robust & Training-Free \\
\midrule
NL dialogue multi-agent & RoCo \cite{Mandi2023RoCo}, AutoGen \cite{Wu2023AutoGen}, MetaGPT \cite{Hong2023MetaGPT}, Blackboard \cite{Salemi2025Blackboard} & Mesh / shared state & No & Weak & Yes \\
Centralized / hierarchical planning & COHERENT \cite{Liu2025COHERENT}, SMART-LLM \cite{Kannan2023SMART} DT-based method \cite{Xu2023digital,Tang2023digital,Zhang2021Scheduling}& Star $\leftrightarrow$ center & Yes & Depends on strong center & Yes \\
Learned multi-agent comm. & CommFormer \cite{Hu2024CommFormer}, GACG \cite{duan2024gacg}, ML-Comm \cite{Ding2024MultiLevelComm} & Learned topology & Implicit & Needs retraining & No \\
Distributed consistency / consensus & Paxos \cite{Lamport1998Paxos}, Secure Consensus \cite{ru2025secureconsensus,liu2025doublelayerconsensus} & Peer network & N/A & Homogeneous nodes & Yes \\
\textbf{LDT-Coord (ours)} & \textbf{This paper} & \textbf{Star $\leftrightarrow$ DT} & \textbf{Yes} & \textbf{Strong} & \textbf{Yes}\\
\bottomrule
\end{tabular}
\end{table*}
\section{System Model and Problem Formulation}
\label{sec:system_model}

\begin{figure}[t]
  \centering
  \includegraphics[width=\linewidth]{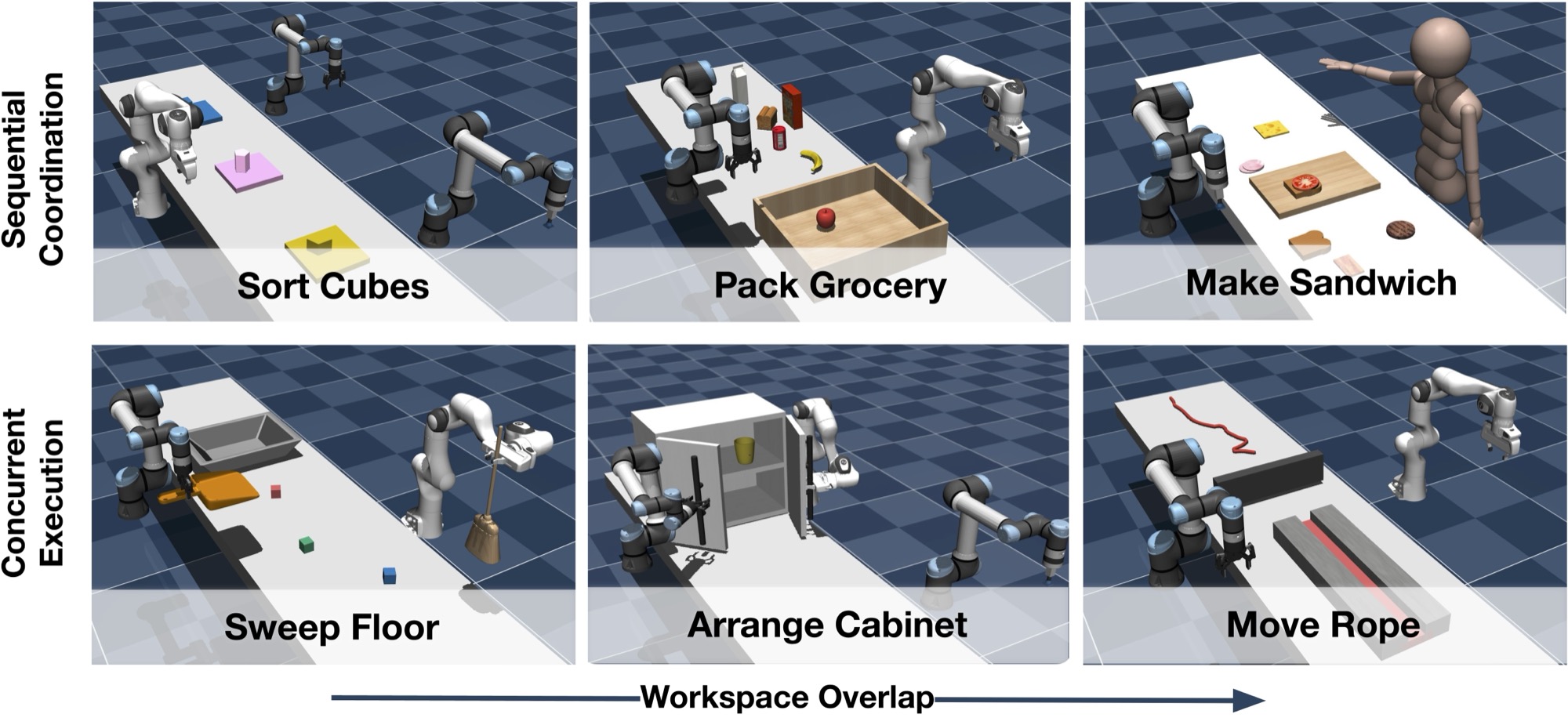}
  \caption{The six multi-arm collaboration tasks considered in this work, spanning sequential coordination and concurrent execution under increasing workspace overlap~\cite{Mandi2023RoCo}.}
  \label{fig:tasks}
\end{figure}
\begin{figure}[t]
  \centering
  \includegraphics[width=\linewidth]{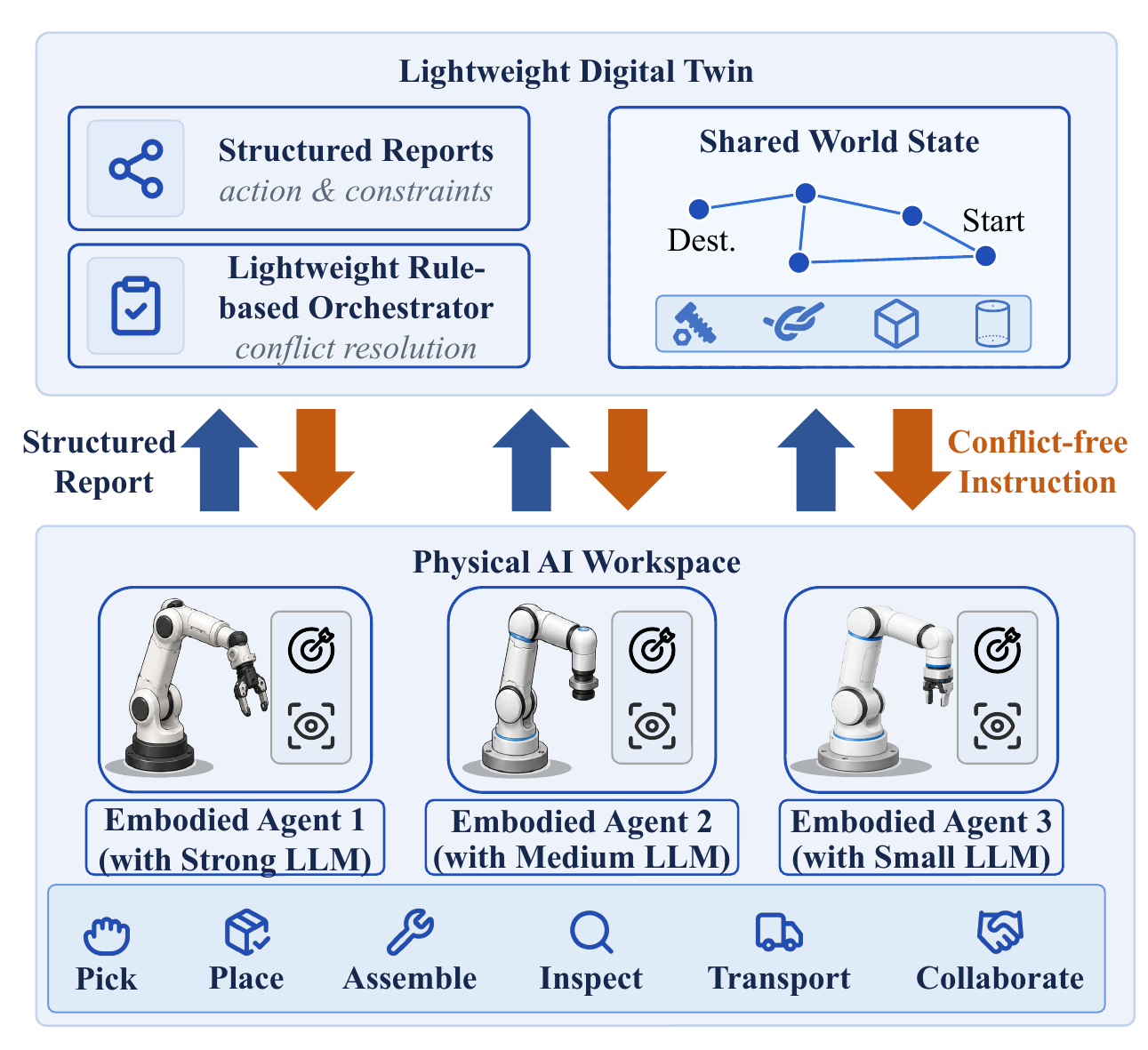}
  \caption{Illustration of the considered system model.}
  \label{fig:system_overview}
\end{figure}

We consider a collaborative team of $n$ heterogeneous LLM-driven embodied agent arms $\mathcal{N}=\{1,\dots,n\}$ that jointly complete a task $g$ in a shared physical workspace (e.g., sort cubes, pack groceries) as shown in Fig. \ref{fig:tasks}.
Each arm $i$ is powered by an LLM $M_i$ whose capability may differ across agents and that can decide its own action individually based on observing only a local projection of the physical state. 
To enable efficient collaboration between arms, a coordination middleware DT $\mathcal{D}$ is introduced.
It derives all of its state observations from structured reports on arms and issues conflict-free plans to agents through a rule-based policy. 
Specifically, at each discrete time step $t$, each arm and $\mathcal{D}$ execute the closed-loop of the following four stages, as shown in Fig. \ref{fig:system_overview}.

\begin{enumerate}
  \item \textbf{State acquisition.} Each arm $i$ obtains its observable state and concatenates it with the task objective $g$ to form its local context.
  \item \textbf{Action selection.} Each arm $i$  autonomously selects an action within its legal action set based on its local context. 
  It also generates the constraint declaration that the selected action imposes on the operated objects (e.g., whether the action requires other arms to cooperate with or to avoid it).
  \item \textbf{Intention reporting.} Each arm $i$ packages its intended action together with the constraint declaration into a structured message and reports it to the $\mathcal{D}$. Each arm $i$ executes the selected action simultaneously.
  \item \textbf{Conflict avoidance.} 
  After aggregating the received messages, $\mathcal{D}$ computes the maximal consistent executable action set following conflict-avoiding orchestration rules of the $n$ arms in this slot. 
  Then, DT sends termination instructions to the vetoed arms so as to avoid conflicts.
\end{enumerate}

These four stages repeat until the task is completed.

\subsection{Local State Acquisition}
\label{subsec:state}

Coordination among the arms takes place in a shared physical workspace with a set of operable entities $\mathcal{E}=\{1,\cdots,E\}$.
Each object $e\in\mathcal{E}$ position on a three-dimensional position $x_t(e)\in\mathbb{R}^3$ at time slot $t$.
The task objective $g$ is to gradually transport every object from its initial position to its respective target position $x_e^\star\in\mathbb{R}^3$.
We introduce the discrete placement points where an object can rest into a set of shared resources (e.g., workstations, handover panels, slots, and other placement positions). 
This placement set is denoted as $\mathcal{P}$, where each resource $p\in\mathcal{P}$ corresponds to a reference three-dimensional position $\bar x_p$. 
For transport tasks that require several relay steps, the resources naturally carry an order along the transport flow.
We record the step position of resource $p$ in this flow as the integer $\mathrm{pos}(p)\in\mathbb{Z}$. 
We can also define the resource currently occupied by object $e$ is the resource closest to its position, given by $\ell_t(e)=\arg\min_{p\in\mathcal{P}}\|x_t(e)-\bar x_p\|$. 
Symmetrically, its target resource is the resource closest to the target position $x_e^\star$, derived as $p_e^\star=\arg\min_{p\in\mathcal{P}}\|x_e^\star-\bar x_p\|$. 
One transport task thus amounts to $\ell_t(e)$ advancing along the order toward $p_e^\star$ until the endpoint is reached. 
Besides the position state, we further denote logical condition state (e.g., cabinet door open or closed) as $u_k\in\{0,1\}$, where $u_k=1$ means the condition holds, and $u_k=0$ otherwise.
All logical conditions form the set $\mathcal{X}$, whose true members at time $t$ form the current logical state set $W_t\subseteq\mathcal{X}$.
Each arm $i$'s reach band $\mathcal{B}_i\subseteq\mathcal{P}$ is the subset of resources that its end effector can touch, with its gripper holding at most one object per step. 
Then we can define the shared band reachable by two or more arms as
\begin{equation}
\mathcal{P}^{\mathrm{sh}}=\{p\in\mathcal{P}\mid |\{i\mid p\in\mathcal{B}_i\}|\ge 2\}.
\end{equation}
$\mathcal{P}^{\mathrm{sh}}$ can be the panels on which arms hand objects over in relay sorting, and the common grasp positions of two arms co-lifting a rope.
The local observation of each agent $i$ is given by a deterministic partially observable projection operator $\Pi_i$. 
This operator keeps only the objects falling within its reach band that have not yet reached their target resource, together with the resources they occupy, which is expressed as
\begin{equation}
\Pi_i(s_t)=\big\{(e,\ell_t(e),W_t)\mid \ell_t(e)\in\mathcal{B}_i,\ \ell_t(e)\neq p_e^\star\big\}.
\end{equation}
The observation of arm $i$ used for action decision is the local context and the task objective, which is given by 
\begin{equation}
c_{i,t}=(\Pi_i(s_t),g),
\end{equation}
where $s_t$ is the global physical state at time $t$ that aggregates the three-dimensional positions of the above objects, the gripper holdings, and the logical conditions.
The task objective $g=\{\varsigma_e\}_{e\in\mathcal{E}^\star}$ is a set of per-object success criteria, where $\mathcal{E}^\star\subseteq\mathcal{E}$ collects the objects involved in the task objective. 
Whether object $e$ is transported successfully is judged directly by whether its three-dimensional position has reached the expected position, which is given by
\begin{equation}
\varsigma_e(s_t)=\begin{cases}1,&\|x_t(e)-x_e^\star\|\le\epsilon_{\mathrm{pos}},\ \forall e\in\mathcal{E}^\star,\\[2pt]0,&\text{otherwise},\end{cases}
\end{equation}
where $\epsilon_{\mathrm{pos}}$ tolerance ball around its target position $x_e^\star$.

To further describe the actions of the embodied robots, we decompose the multi-embodied-robot collaborative task into the three most basic classes of atomic tasks\footnote{In this paper an atomic task refers to the most basic and indivisible temporal coordination relation between two shared-resource actions, and this decomposition of the pair of basic control-flow relations follows the characterization of workflow patterns \cite{vanderAalst2003WorkflowPatterns}.}, which named mutual exclusion, synchronization, and dependency as follows.
\begin{enumerate}
  \item \textbf{Mutual exclusion (e.g., sorting).} 
  The task objective is to move each entity toward its target position via several cross-arm relays. 
  When two arms contend for the entity on the same panel in the same slot, a mutual-exclusion conflict arises, so that $\mathcal{D}$ admits the agent with higher priority to avoid this conflict.
  \item \textbf{Synchronization (e.g., rope co-lifting).} 
  The task objective is to lift the different sides of the object by the arms simultaneously and place them at a designated position. 
  In each slot, the arms must ensure that the other arms are in the same lifting group. 
  \item \textbf{Dependency (holding to fetch).} 
  The task objective is to take a target object under a specific state.
  To achieve this, one arm maintains the required logical state, while the other fetches the target object.
  The fetching action is valid only when the required logical state holds (i.e., $u_k=1$).
  In each slot, the state-maintaining arm produces and maintains the logical state $u_k$, while the fetching arm declares a persistent dependency on this state.
  Then $\mathcal{D}$ admits the fetching action only when another arm maintains the required logical state in the same slot, which imposes an enforced ordering with in-between timing.

\end{enumerate}

\subsection{Autonomous Action Selection}
\label{subsec:action}

In each slot, every embodied robotic arm must select a suitable action according to its local observation $c_{i,t}$. 
The full action space is $\mathcal{O}=\{a_0\}\cup(\mathcal{E}\times\mathcal{P})$, where $a_0$ denotes staying still, that is, initiating no action in this slot. The pair $(e,p)$ denotes grasping entity $e$ and placing it onto resource $p$. 
Constrained by partial observability and the reach band, the legal action set of the arm is given by a deterministic mask $\mathcal{O}_i^{\mathrm{legal}}(s_t)=\{a_0\}\cup\{(e,p)\mid\ell_t(e)\in\mathcal{B}_i,\ p\in\mathcal{B}_i,\ \nu(e,p)=1\}$. 
Here, the task-validity predicate $\nu\in\{0,1\}$ is given explicitly by
\begin{equation}
\begin{aligned}
\nu(e,p)&=\mathds{1}\Big[\,p=p_e^\star\ \lor\ \big(p\in\mathcal{P}^{\mathrm{sh}}\ \\
& \wedge\ |\mathrm{pos}(p)-\mathrm{pos}(p_e^\star)|<|\mathrm{pos}(\ell_t(e))-\mathrm{pos}(p_e^\star)|\big)\,\Big].
\end{aligned}
\end{equation}
$\nu(e,p)=1$ imply that move object $e$ to plane $p$ is available, and $\nu(e,p)=0$ otherwise. 
$\nu(e,p)=1$ if $p$ is exactly the target resource of $e$, or $p$ is a shared relay resource that lies strictly closer to the target $p_e^\star$ in order distance $\mathrm{pos}$.
The reference is the distance of the resource $\ell_t(e)$ currently occupied by $e$ to $p_e^\star$, and $\nu(e,p)=0$ otherwise. 
On this legal action set, arm $i$ further uses its own LLM $M_i$ together with the local context $c_{i,t}$ to score each feasible action. 
It takes the highest-scoring one as the action selected in this step, which is given by
\begin{equation}
a_{i,t}=\arg\max_{a\in\mathcal{O}_i^{\mathrm{legal}}(s_t)}\ \frac{1}{|a|}\sum_{k}\log p_{M_i}\!\left(a_k\mid c_{i,t}\right),
\end{equation}
where $p_{M_i}(\cdot\mid c_{i,t})$ is the conditional token likelihood of the LLM $M_i$ and $|a|$ is the action token length. Since the scoring is performed over the enumerated legal set rather than by free generation, no illegal action can be produced. The action depends only on $M_i$ and $c_{i,t}$, and heterogeneity is introduced by the capability differences among $\{M_i\}$.
After selecting the suit action $a_{i,t}$, each arm $i$ executes the selected action immediately.

\subsection{Structured Intention Reporting}
\label{subsec:report}

After selecting its own action, each arm $i$ must further impose constraints on the objects to be operated on to ensure that the task can be completed correctly. 
To this end, we design three principles to address the respective coordination needs of the mutual-exclusion, synchronization, and dependency atomic tasks.

\textbf{Mutual-exclusion field.} Coordination $\xi_{i,t}\subseteq\mathcal{E}\cup\mathcal{P}$ lists the entities and resources that this step's action needs to occupy exclusively. 
This means that no two admitted actions may occupy the same element of $\xi$ in the same slot.
It grounds the relation that two actions cannot proceed at the same time into an occupancy declaration on concrete shared elements.

\textbf{Synchronization field.} Coordination $\zeta_{i,t}\subseteq\mathcal{N}$ lists the members of the joint group to which this step's action belongs, namely the roster of initiators of a set of actions that must execute simultaneously. 
The group takes effect as a whole if and only if the group members consistently name each other in this slot and all their actions are admitted; otherwise, a collision occurs.
This expresses the relation that they must act simultaneously.

\textbf{Dependency field.} An arm characterizes its production of and demand for the shared logical states $\mathcal{X}$ with two pairs of state fields. On the supply side, $\psi_{i,t}\subseteq\mathcal{X}$ is the state instantaneously produced by this step's action and $\chi_{i,t}\subseteq\mathcal{X}$ is the state continuously held by this step's action. On the demand side, $\alpha_{i,t}\subseteq\mathcal{X}$ is the prerequisite state that this step's action needs. This state must be produced in the same step by some action of this step, that is an instantaneous dependency. Correspondingly, $\omega_{i,t}\subseteq\mathcal{X}$ is the state that this step's action needs and that must be held in the same step by some action of this step, that is a persistent dependency. This pair of demand and supply declarations expresses the relation that ordering and in-between timing must hold.

Combining the above information, we define the constraint declaration set as
\begin{equation}
\mathcal{C}_{i,t}=\big\{\xi_{i,t},\ \zeta_{i,t},\ \alpha_{i,t},\ \omega_{i,t},\ \psi_{i,t},\ \chi_{i,t}\big\}.
\end{equation}
Based on $\mathcal{C}_{i,t}$ arm $i$ can clarify its needs for the current action and then packages together with action $a_{i,t}$ into a structured message $m_{i,t}=\langle a_{i,t},\mathcal{C}_{i,t}\rangle$ and reports it to $\mathcal{D}$ for avoiding collisions.
In our proposed scheme, the agent message is carried as structured key-value pairs rather than the traditional NL-based agent dialogue strategy. The DT $\mathcal{D}$ can thus obtain information that is stripped of the reporter's language ability.
Further accounting for the limited communication resources, we define $\sigma_{i,t}\in\{0,1\}$ as an indicator to account for whether this report takes place in the time slot $t$.
Here $\sigma_{i,t}=1$ means that arm $i$ has one uplink or downlink interaction with $\mathcal{D}$ at $t$ and $\sigma_{i,t}=0$ means no interaction. Accordingly, the set of agents that report structured information at time $t$ is denoted as $\mathcal{U}_t=\{i\mid\sigma_{i,t}=1\}$.

To carry out reporting, we consider a shared uplink. The end effector of arm $i$ is located at $q_i\in\mathbb{R}^3$ and the DT $\mathcal{D}$ is deployed at a fixed position $q_{\mathcal{D}}\in\mathbb{R}^3$. The uplink channel gain between the two attenuates with distance according to path loss, given by $h_i=h_0\,\|q_i-q_{\mathcal{D}}\|^{-\mu}$, where $h_0$ is the reference gain and $\mu>0$ is the path-loss exponent. With transmit power $P_i$, link bandwidth $B$, and noise power spectral density $N_0$, the achievable uplink rate from arm $i$ to $\mathcal{D}$ is
\begin{equation}
v_{i,t}=B\log_2\!\Big(1+\frac{P_i\,h_i}{N_0\,B}\Big).
\end{equation}
The per-message delay needed to report one structured message $m_{i,t}$ can be expressed as
\begin{equation}
\delta_{i,t}(q_i,|m_{i,t}|,P_i,B)=\frac{|m_{i,t}|}{B\log_2\!\Big(1+\dfrac{P_i\,h_0\,\|q_i-q_{\mathcal{D}}\|^{-\mu}}{N_0\,B}\Big)},
\end{equation}
where $|m_{i,t}|$ is the bit length. All reports share the same uplink to $\mathcal{D}$. The total delay of coordination communication in this slot is therefore the sum of the delays of all reporters in this step, given by
\begin{equation}
D_t=\sum_{i\in\mathcal{U}_t}\delta_{i,t}.
\end{equation}

\subsection{Coordination and Execution at the Digital Twin}
\label{subsec:resolve}

After aggregating the messages $\{m_{i,t}\}$ of all acting arms, the DT $\mathcal{D}$ does not need to understand the task semantics. It relies only on the declared constraints $\{\mathcal{C}_{i,t}\}$ and the current logical state set $W_t$ to act. From the acting set $\mathcal{N}_t^{\mathrm{act}}=\{i\mid a_{i,t}\neq a_0\}$, it selects by rules a maximal consistent arm set $\mathcal{A}_t^\star$ whose actions do not conflict with each other, so as to avoid conflicts. To characterize consistency, we first consider any acting-arm subset $\mathcal{A}\subseteq\mathcal{N}_t^{\mathrm{act}}$, each member $i$ of which carries this step's intended action $a_{i,t}$. The action of arm $i$ should be consistent with the actions of the other members of the set $\mathcal{A}$ and with the current state $W_t$. This holds when it simultaneously passes the three checks of mutual exclusion, synchronization, and dependency. We next give the consistency subconditions for these three coordination relations separately, and then conjoin them into a consistency criterion $\eta(i\mid\mathcal{A},W_t)\in\{0,1\}$.

\textbf{Mutual-exclusion consistency.} Arm $i$ contends for no exclusive element with any other member of $\mathcal{A}$, given by
\begin{equation}
\eta_{\xi}(i\mid\mathcal{A})=\mathds{1}\big[\,\forall j\in\mathcal{A}\setminus\{i\}:\ \xi_{i,t}\cap\xi_{j,t}=\varnothing\,\big].
\end{equation}

\textbf{Synchronization consistency.} The members of the joint group $\zeta_{i,t}$ to which arm $i$ belongs all fall within $\mathcal{A}$ and name the same group as each other, given by
\begin{equation}
\eta_{\zeta}(i\mid\mathcal{A})=\mathds{1}\big[\,\zeta_{i,t}\subseteq\mathcal{A}\ \wedge\ \forall m\in\zeta_{i,t}:\ \zeta_{m,t}=\zeta_{i,t}\,\big].
\end{equation}

\textbf{Dependency consistency.} The instantaneous prerequisite states needed by arm $i$ are produced in this step by some action in $\mathcal{A}$ or are already true in $W_t$. Likewise, the persistent states it needs are held in this step by some action in $\mathcal{A}$ or are already true in $W_t$, given by
\begin{equation}
\begin{aligned}
\eta_{\omega}(i\mid\mathcal{A},W_t)=\mathds{1}\big[\,\alpha_{i,t}\subseteq\textstyle\bigcup_{j\in\mathcal{A}}(\psi_{j,t}\cup \chi_{j,t})\\
\cup W_t\ \wedge\ \omega_{i,t}\subseteq\textstyle\bigcup_{j\in\mathcal{A}}\chi_{j,t}\cup W_t\,\big].
\end{aligned}
\end{equation}

The action of arm $i$ is consistent with $\mathcal{A}$ only when the three checks pass simultaneously, so the consistency criterion is the conjunction of the three, given by
\begin{equation}
\eta(i\mid\mathcal{A},W_t)=\eta_{\xi}(i\mid\mathcal{A})\ \wedge\ \eta_{\zeta}(i\mid\mathcal{A})\ \wedge\ \eta_{\omega}(i\mid\mathcal{A},W_t),
\end{equation}
which takes the value $1$ if and only if the mutual-exclusion, synchronization, and dependency declarations of arm $i$ are satisfied at the same time.
Accordingly, the executable set optimization method is given by
\begin{align}
\mathcal{A}_t^\star&=\operatorname*{arg\,max}_{\mathcal{A}\subseteq\mathcal{N}_t^{\mathrm{act}}}\ |\mathcal{A}|,\\
\label{eq:objectiveAction} 
&\text{s.t.}\quad \eta(i\mid\mathcal{A},W_t)=1,\ \forall i\in\mathcal{A}.
\end{align}
The subset of maximal cardinality under the premise that every member of the set satisfies its own declared constraints. 
Consider an agent arm $i\notin\mathcal{A}_t^\star$ whose action is vetoed by the DT during optimziation of $\mathcal{D}$. 
To this agent, $\mathcal{D}$ sends back an short instruction $d_{i,t}\in\{d_0,d_1\}$, where $d_0$ means yield and stop and $d_1$ means wait. The vetoed party backs off autonomously while $\mathcal{D}$ does not specify a replacement action, so that the arm generates an action again at the next step. The admitted set is realized as a physical advance through a shared execution map $\Phi$, given by $s_{t+1}=\Phi(s_t,\{a_{i,t}\}_{i\in\mathcal{A}_t^\star})$, and this map is consistent for all arms.

\subsection{Problem Formulation}
\label{subsec:problem}

We formalize an optimization problem whose goal is to maximize the expected cumulative task-objective reward under the communication constraint, which is formulated as
\begin{align}
\max_{\pi^c,\,\mathcal{U}_t}\quad & \mathbb{E}\!\left[\sum_{t=0}^{T}\gamma^t R_t\right]\label{eq:objective}
\\
\text{s.t.}\quad & D_t\le \bar D,\ \forall t,\tag{\theequation a}\label{eq:c1}\\
& \eta(i\mid\mathcal{A}_t^\star,W_t)=1,\ \forall i\in\mathcal{A}_t^\star,\ \forall t,\tag{\theequation b}\label{eq:c2}
\end{align}
where $R_t=\sum_{e\in\mathcal{E}^\star}(\varsigma_e(s_{t+1})-\varsigma_e(s_t))$ is the per-step reward.
The decision variable is the communication-selection policy $\pi^c$, while the per-arm action and constraint declaration $(a_{i,t},\mathcal{C}_{i,t})$ are produced locally by each $M_i$ and are treated as given rather than optimized.
The policy $\pi^c$ decides the per-step reporting indicators $\{\sigma_{i,t}\}$ and thereby decides the reporting set $\mathcal{U}_t$. Here $\eta(\cdot)$ is the declaration-satisfaction criterion defined above, $\gamma\in(0,1]$ is the discount on future-action reward, $T$ is the time horizon, and $\bar D$ is the per-step deadline of coordination communication.
(\ref{eq:objective}a) is the per-step coordination-delay constraint, which requires that the total delay of all reports of this step on the shared uplink not exceed the deadline $\bar D$. 
This constraint captures the real-time requirement of a constrained link under heterogeneous deployment.
(\ref{eq:objective}b) is the consistency constraint, which captures the coordination relations that physically cannot be violated simultaneously and guarantees that the action set has no explicit conflict.


Directly solving the above problem faces three structural challenges. First, heterogeneity-induced declaration noise. Since $\mathcal{C}_{i,t}$ is generated by $\pi_i$ through $M_i$, weak LLMs may omit or misstate declarations. This may cause local violations of (\ref{eq:objective}b). The coordination mechanism should therefore tolerate individual declaration errors, rather than being bottlenecked by the weakest arm. Second, cross-atomic-task coupling creates cascading conflicts. Exclusive, grouped, and dependent constraints are coupled through shared actions and states. 
The incomplete grouped action may then remove a prerequisite state for a dependent action.
Thus, conflicts propagate over the constraint graph. 
Third, partial observability prevents global planning. 
Since each arm only observes through $\Pi_i$, neither any arm nor $\mathcal{D}$ has complete access to $s_t$. 
Task ordering must instead emerge from local decisions and the reachable topology. These challenges motivate a lightweight conflict-avoidance mechanism that is robust to declaration noise, iteratively convergent, and able to guarantee finite-step progress without centralized planning.

\section{Proposed Method}
\label{sec:method}



\subsection{Unified Lightweight Orchestrator}
\label{subsec:orchestrator}

To solve the optimization problem formulated in (\ref{eq:objective}), we unify the mutual-exclusion, synchronization, and dependency atomic tasks into a single orchestrator $\mathcal{R}$ of four typed rules that avoids declaration conflicts.

We first define the mutual-exclusion rule, $T_{\xi}$, which handles the scenario in which several admitted actions contend for the same exclusive resource. Among the actions that contend for the same element of $\xi$ it keeps only the one with the highest yielding priority and removes the rest. This enforces the constraint that two actions cannot occupy the same resource at the same time. It is defined as
\begin{equation}
\begin{aligned}
T_{\xi}(\mathcal{A})&=\mathcal{A}\setminus\big\{i\in\mathcal{A}\mid\ \exists\,\\
&r\in\xi_{i,t},\ \exists\,j\in\mathcal{A}\setminus\{i\},\ r\in\xi_{j,t},\ \rho(a_{j,t})\succ\rho(a_{i,t})\big\},
\label{eq:Txi}
\end{aligned}
\end{equation}
where $\rho(a)\in \{0,1\}\times\mathbb{Z}_{\le 0}\times\mathbb{Z}_{\le 0}$ is the yielding priority which is defined as
\begin{equation}
\rho(a)=\Big(\,\mathds{1}[\,p=p_e^\star\,],\ -\,\big|\mathrm{pos}(p)-\mathrm{pos}(p_e^\star)\big|,\ -\,\kappa(i)\,\Big).
\label{eq:priority}
\end{equation}
The first component $\mathds{1}[\,p=p_e^\star\,]$ is the terminal bit that marks whether the drop slot $p$ is exactly the target resource of object $e$. An action that reaches the final state directly is more worth keeping than one that only relays. The second component is the negated remaining step count $\big|\mathrm{pos}(p)-\mathrm{pos}(p_e^\star)\big|$, that is, the distance of the drop slot $p$ to the target resource $p_e^\star$ along the ordinal $\mathrm{pos}$, where a smaller distance is more worth keeping. 
In the third component, quantity $\kappa(i)$ is the fixed index of the arm, used only to break ties when the first two components are equal. 
This priority is built solely from the local features of a single action and needs no inter-arm negotiation, which keeps $\mathcal{D}$ lightweight.
For two actions that contend for the same resource, the components are compared from the first to the third, and on the first component where the values differ, the larger value has the higher priority. 
Accordingly, $T_{\xi}$ keeps only the highest-priority contender and removes the others.

Similarly, we can define the synchronization rule $T_{\zeta}$ that handles the scenario in which a group of actions that must execute simultaneously is not complete.
The group is kept only when its members name the same group at this step, and all of them still remain in $\mathcal{A}$, and otherwise the whole group is removed. This rules out half-finished execution. It is defined as
\begin{equation}
\begin{aligned}
T_{\zeta}(\mathcal{A})=\mathcal{A}\setminus\big\{i\in\mathcal{A}\mid\ &\zeta_{i,t}=G\neq\varnothing,\\
&\ \neg\big(\forall m\in G:\ m\in\mathcal{A}\wedge\zeta_{m,t}=G\big)\big\}.
\end{aligned}
\label{eq:Tzeta}
\end{equation}

We define the instantaneous-dependency rule $T_{\alpha}$ that handles the scenario in which the precondition that an action requires is produced by no one at this step. 
For this, we first define the production closure and the holding closure of the admitted set, which are
\begin{equation}
\bar\psi(\mathcal{A})=\bigcup_{j\in\mathcal{A}}(\psi_{j,t}\cup\chi_{j,t}),\qquad \bar\chi(\mathcal{A})=\bigcup_{j\in\mathcal{A}}\chi_{j,t}.
\label{eq:closures}
\end{equation}
Specifically, all states produced and held by the admitted set at this step. The rule removes the actions whose precondition is neither in the production closure nor already true in $W_t$ which is defined as
\begin{equation}
T_{\alpha}(\mathcal{A})=\mathcal{A}\setminus\big\{i\in\mathcal{A}\mid\ \alpha_{i,t}\not\subseteq\bar\psi(\mathcal{A})\cup W_t\big\}.
\label{eq:Talpha}
\end{equation}

The persistent-dependency rule $T_{\omega}$ handles the scenario in which a state that an action requires to be held at this step is held by no one. It removes the actions whose required state is neither in the holding closure nor already true, which enforces that the action must execute while some state is being held. It is defined as
\begin{equation}
T_{\omega}(\mathcal{A})=\mathcal{A}\setminus\big\{i\in\mathcal{A}\mid\ \omega_{i,t}\not\subseteq\bar\chi(\mathcal{A})\cup W_t\big\}.
\label{eq:Tomega}
\end{equation}

Building on these four rules, the DT $\mathcal{D}$ composes them through the unified orchestrator $\mathcal{R}$ and applies the composition to the candidate set repeatedly until it converges to a self-consistent executable set. The output action set then realizes the goal of conflict avoidance, that is, it selects from all intended actions of this step a subset that are mutually compatible and can execute together. This executable set is given by
\begin{equation}
\mathcal{A}_t^\star=\mathcal{R}(\{m_{i,t}\})=\mathrm{iter}\big(T_{\xi}\circ T_{\zeta}\circ T_{\alpha}\circ T_{\omega}\big)(\mathcal{N}_t^{\mathrm{act}}),
\label{eq:orchestrator}
\end{equation}
where $\mathrm{iter}(\cdot)$ denotes repeated application of the composite rule in the parentheses until the result no longer changes, and each rule acts on the candidate set $\mathcal{A}$ and removes the offenders of its type.

\begin{remark}
Note that at the fixed point, no rule fires, so every retained arm is consistent and the output is always feasible for (\ref{eq:objectiveAction}). 
In the single-resource atomic tasks considered here, the mutual-exclusion conflicts form disjoint cliques while incomplete synchronization groups and unmet dependencies are removed deterministically, so the iteration keeps one arm per contended resource plus all self-supporting arms and thus attains the maximum cardinality; in the general coupled case it remains a linear-time feasible solver with the priority $\rho$ as a deterministic tie-break.
\end{remark}

After $\mathcal{A}_t^\star$ is obtained, $\mathcal{D}$ returns to each vetoed arm $i\in\mathcal{N}_t^{\mathrm{act}}\setminus\mathcal{A}_t^\star$ a very short downlink instruction $d_{i,t}\in\{d_0,d_1\}$, where $d_0$ is the yield-and-stop instruction and $d_1$ is the wait instruction. The actual value is determined by which rule removed the action. If $i$ yields in a contention by priority $\rho$ under the mutual-exclusion rule $T_\xi$, then $d_0$ is returned, because the contention at this step is irrecoverable. The arm then abandons the action and backs off in place or reselects. If $i$ is removed because its group is incomplete under synchronization $T_\zeta$, or because its precondition is not yet satisfied under the dependency rules $T_\alpha$ or $T_\omega$, then $d_1$ is returned. In this case the missing condition may be met in a later slot, and the arm keeps its current intention and waits. Neither instruction prescribes a replacement action for the arm. $\mathcal{D}$ only states that this step is infeasible and of which nature, whether to yield or to wait, while the back-off and the reselection are still carried out by the arm itself.

Since each of the four rules only removes elements from $\mathcal{A}$, so that $T(\mathcal{A})\subseteq\mathcal{A}$, their composition is monotonically shrinking and its finite-step convergence is guaranteed by the following proposition.

\begin{proposition}[Finite-step convergence of the orchestrator]
\label{prop:convergence}
The composition $\Theta=T_{\xi}\circ T_{\zeta}\circ T_{\alpha}\circ T_{\omega}$ satisfies $\Theta(\mathcal{A})\subseteq\mathcal{A}$, so the iteration that starts from $\mathcal{N}_t^{\mathrm{act}}$ must converge in a finite number of steps to a unique stable executable set, and it is reached in at most $|\mathcal{N}_t^{\mathrm{act}}|$ rounds.
\end{proposition}
\begin{proof}
    See Appendix~\ref{app:convergence}.
\end{proof}




The closure and completeness of the three atomic tasks over the coordination space are characterized by the following proposition.

\begin{proposition}[Closure and completeness of the three atomic tasks]
\label{prop:closure}
In a discrete-slot scenario where actions are constrained by shared resources and logical states, and where coordination constraints never force any arm to act, so that both arms may remain idle, any coordination constraint between two interacting actions is logically equivalent to a conjunction of mutual-exclusion and dependency atomic tasks. Synchronization is precisely the symmetric conjunction of two oppositely directed dependencies. Therefore, the atomic-task set mutual exclusion, synchronization, dependency, grounded by the four rules $T_\xi,T_\zeta,T_\alpha,T_\omega$, is closed under conjunction, $n$-ary extension, and interval extension, and spans all coordination constraints under this model. These three relations are pairwise distinct: they require disjoint, simultaneous, and ordered execution, respectively, and none is a special case of another~\cite{vanderAalst2003WorkflowPatterns}.
\end{proposition}

\subsection{Learned Communication-Selection Layer}
\label{subsec:comm}

The coordination layer above already achieves reliable cooperation among the arms without any training. However, every arm still has to upload its action and constraint declaration to $\mathcal{D}$ at every step, and these reports accumulate a non-negligible per-step latency $D_t$ on a heterogeneous and constrained shared uplink. To further reduce the transmission overhead, this subsection solves the communication-selection policy $\pi^c$ of the formulated problem. Since the latency deadline (\ref{eq:objective}a) is coupled step by step in time and its decision observation $z_{i,t}$ is only a partially observable proxy, it is hard to optimize $\{\sigma_{i,t}\}$ directly. To address this we treat it as a sequential decision problem and model it as a constrained partially observable Markov decision process (C-POMDP). 
Its state is the decision observation $z_{i,t}$, its action is the report indicator $\sigma_{i,t}$, and its instantaneous reward is the team-level shaped signal $r_t$. It maximizes $\mathbb{E}[\sum_t\gamma^t r_t]$ under the expectation form $\mathbb{E}[D_t]\le\bar D$ of (\ref{eq:objective}a), where $D_t=\sum_{i\in\mathcal{U}_t}\delta_{i,t}$ is the total report latency of this step and $\mathcal{U}_t=\{i\mid\sigma_{i,t}=1\}$ is the report set of this step. This layer is strictly decoupled from the coordination layer. The coordination layer, namely the lightweight $\mathcal{D}$ together with the orchestrator, stays training-free and unchanged. The policy $\pi^c$ only decides whether to report, while the reports that arrive are still arbitrated by the same orchestrator $\mathcal{R}$.

Then, we define the C-POMDP of this communication-selection subproblem element by element. 
The agent on each active arm $i$ is the decision unit that decides whether to report. 
The observation $z_{i,t}$ is built only from quantities that arm $i$ can read out locally, and it contains neither the information of other arms nor any task semantics. It is
\begin{equation}
z_{i,t}=\big(b_{i,t},\,\|q_i-q_{\mathcal{D}}\|,\,h_i,\,\Delta_{i,t},\,\tau_{i,t}\big),
\label{eq:obs}
\end{equation}
where $b_{i,t}\in\{0,1\}$ is the conflict-possibility proxy that indicates whether the resource that $i$ declares to occupy at this step falls on a shared resource $\mathcal{P}^{\mathrm{sh}}$, that is, whether the action touches a position that another arm may contend for. This judgment depends only on the static shared-resource structure $\mathcal{P}^{\mathrm{sh}}$ and on the own occupancy declaration $\xi_{i,t}$ of $i$, and it needs no information of any other arm. The end-effector distance to the twin $\|q_i-q_{\mathcal{D}}\|$ and the channel gain $h_i=h_0\|q_i-q_{\mathcal{D}}\|^{-\mu}$ jointly characterize the local link state of arm $i$. The quantity $\Delta_{i,t}\in\{0,1\}$ is the intention novelty, namely whether the action of this step is the same as the previous step. The quantity $\tau_{i,t}\in\mathbb{Z}_{\ge 0}$ is the number of consecutive rounds that arm $i$ has not reported since its last report. The action $\sigma_{i,t}\in\{0,1\}$ is the report indicator. At each step every agent first produces an intended action $a_{i,t}$ on its own, and then $\pi^c$ outputs $\sigma_{i,t}$ from $z_{i,t}$. Let the report set be $\mathcal{U}_t=\{i\mid\sigma_{i,t}=1\}$. The orchestrator then computes the consistent set only over $\mathcal{U}_t$ as $\mathcal{A}_t^\star=\mathcal{R}(\{m_{i,t}\}_{i\in\mathcal{U}_t})$. The action of an arm $i\notin\mathcal{U}_t$ that does not report does not enter conflict avoidance and lands directly through the execution map $\Phi$ under its own intention. It bypasses the veto, so that its conflict is only guarded by the physical gate. The communication cost is the total latency of this step on the shared uplink, $D_t=\sum_{i\in\mathcal{U}_t}\delta_{i,t}$, which under an equal channel is proportional to the report count $|\mathcal{U}_t|$.

The reward is the team-level per-step signal
\begin{equation}
r_t=R_t-\lambda\,D_t-\beta\,f_t,
\label{eq:reward}
\end{equation}
where $R_t=\sum_{e\in\mathcal{E}^\star}\big(\varsigma_e(s_{t+1})-\varsigma_e(s_t)\big)\ge0$ is the task-goal progress of this step, namely the same per-step reward. Here $\lambda\ge0$ is the latency dual variable, $D_t=\sum_{i\in\mathcal{U}_t}\delta_{i,t}$ is the total report latency of this step, and $f_t$ is the number of physical failures caused by unreported actions at this step. The action of an arm $i\notin\mathcal{U}_t$ that does not report bypasses the orchestration and lands directly through $\Phi$, and whether it is compatible with the concurrent actions of this step determines success or failure. Let the active set that is actually executed concurrently at this step be the union of the admitted arms and the unreported arms that land directly, $\mathcal{A}_t^{\mathrm{exe}}=\mathcal{A}_t^\star\cup(\mathcal{N}_t^{\mathrm{act}}\setminus\mathcal{U}_t)$. 
$f_t=\big|\{\,i\in\mathcal{N}_t^{\mathrm{act}}\setminus\mathcal{U}_t\ \mid\ \eta(i\mid\mathcal{A}_t^{\mathrm{exe}},W_t)=0\,\}\big|$ is the number of actions that land without orchestration yet violate the consistency criterion $\eta$ defined in Section~\ref{sec:system_model}. 
Such incompatibility shows up physically as a collision or a gate failure and is counted by the execution map $\Phi$ at landing. 
Here $\beta\,f_t$ is a reward shaping based on physical execution feedback rather than an injection of the coordination rule. The task progress $R_t$ is sparse, since most steps are zero. The gain of reporting or not reporting is reflected only indirectly through physical execution, so credit is hard to assign. In contrast, $f_t$, as the direct physical consequence of a missed conflict, supplies the sparse reward with a dense and observable signal. When $b_{i,t}{=}1$, not reporting causes a collision, $f_t$ rises, and the return drops, which drives the policy to report when a conflict is possible. When $b_{i,t}{=}0$, not reporting causes neither a collision nor a communication cost, which drives the policy to stay silent when it is safe.

The policy adopts a per-agent Bernoulli parameterization. 
A network $g_\theta$ maps the observation $z_{i,t}$ to a report probability $p_\theta(z_{i,t})=1/\big(1+e^{-g_\theta(z_{i,t})}\big)\in[0,1]$ which is given by
\begin{equation}
\pi^c_\theta(\sigma_{i,t}\mid z_{i,t})=p_\theta(z_{i,t})^{\,\sigma_{i,t}}\big(1-p_\theta(z_{i,t})\big)^{\,1-\sigma_{i,t}}.
\label{eq:policy}
\end{equation}
Here, the probability of reporting with $\sigma_{i,t}=1$ is $p_\theta(z_{i,t})$ and the probability of not reporting with $\sigma_{i,t}=0$ is its complement $1-p_\theta(z_{i,t})$. 
The value is given by a value network $V_\phi$ that estimates from the observation the expected discounted return of that state, $V_\phi(z_{i,t})\approx\mathbb{E}_{\pi^c_\theta}\big[\sum_{k\ge t}\gamma^{k-t}r_k\mid z_{i,t}\big]$. The advantage function is defined as
\begin{equation}
\hat A_t=\hat G_t-V_\phi(z_{i,t}),\qquad \hat G_t=\sum_{k\ge t}\gamma^{k-t}r_k,
\label{eq:advantage}
\end{equation}
where $\hat G_t$ is the Monte Carlo return from $t$ on and $\hat A_t$ measures the gain of the report decision of this step relative to the baseline $V_\phi$. Both are used for the policy improvement and the value fitting of the PPO solver below.

The action of this subproblem is per-agent binary, and the observation is low-dimensional, and after the shaping above, the reward is already densified.
We therefore use PPO-Lagrangian as the solver because its policy gradient naturally couples to the primal--dual ascent of the latency constraint, which is standard modern practice for constrained RL. 
Relaxing the per-step latency constraint (\ref{eq:objective}a) with a multiplier $\lambda\ge0$ yields the saddle-point objective as
\begin{equation}
\min_{\lambda\ge0}\ \max_{\theta}\ \mathbb{E}_{\pi^c_\theta}\!\Big[\sum_t\gamma^t\big(R_t-\beta\,f_t\big)-\lambda\big(\textstyle\sum_t D_t-T\bar D\big)\Big],
\label{eq:saddle}
\end{equation}
which is solved by alternating between its inner and outer levels. The inner level, which fixes $\lambda$
and improves the policy, ascends with PPO. It reuses the advantage $\hat A_t$ and the Monte Carlo return target $\hat G_t$ defined above, replacing single-step bootstrapping with the return-to-go so that the sparse terminal reward propagates back stably over a short horizon. With the probability ratio $\varrho_t(\theta)=\pi^c_\theta(\sigma_{i,t}\mid z_{i,t})/\pi^c_{\theta_{\mathrm{old}}}(\sigma_{i,t}\mid z_{i,t})$ it maximizes the clipped surrogate objective
\begin{equation}
\mathcal{L}^{\mathrm{clip}}(\theta)=\mathbb{E}_t\big[\min\big(\varrho_t(\theta)\,\hat A_t,\ \mathrm{clip}(\varrho_t(\theta),1-\epsilon,1+\epsilon)\,\hat A_t\big)\big],
\label{eq:clip}
\end{equation}
and the value network $V_\phi$ is fitted by $\min_\phi\mathbb{E}_t[(V_\phi(z_{i,t})-\hat G_t)^2]$. The outer level, namely the latency dual, performs a PID primal--dual ascent on $\lambda$. With the per-episode latency excess $\hat g=\frac1T\sum_t D_t-\bar D$ it updates $\lambda\leftarrow[\lambda+k_P\hat g+k_I\sum\hat g+k_D\Delta\hat g]_+$. The proportional term responds quickly, the integral term removes the steady-state excess, and the derivative term suppresses overshoot. The PID form is adapted from the transmission scheduling of goal-oriented semantic communication~\cite{Chen2025Goal}. 
\section{Experiments}
\label{sec:experiments}

For experimentation, we evaluate on the RoCo/MuJoCo physics simulator \cite{Mandi2023RoCo}. The evaluation covers three classes of multi-arm coordination, which correspond one-to-one to the three constraint operators of the orchestrator. 
They are mutual exclusion (e.g., confined-space sorting, three arms relay-sort objects through a shared panel, exclusion $T_{\xi}$), synchronization (move rope, barrier $T_{\zeta}$), and dependency (cabinet, hold-then-fetch, dependency $T_{\omega}$).
The team heterogeneity is induced solely by the capability gap of the per-arm $M_i$. From strong to weak the models are fixed as DeepSeek-R1-Distill-Qwen-1.5B \cite{DeepSeek2025DeepSeekR1} $>$ Qwen2.5-1.5B-Instruct $>$ Qwen2.5-0.5B-Instruct \cite{Qwen2024Technical}. 
Unless otherwise stated, the main experiments use a homogeneous default team that shares Qwen2.5-1.5B. For the heterogeneity-robustness analysis, we additionally form three team configurations with weak, mid, and strong heterogeneity, where a larger capability gap among the per-arm models indicates stronger heterogeneity.
For comparison, we consider four baselines, grouped into a dialogue-based group and a centralized group as follows.

\begin{itemize}
  \item \textbf{RoCo-NL} \cite{Mandi2023RoCo} is the representative dialogue-based paradigm, where multiple arms negotiate the division of labor through several rounds of natural-language (NL) dialogue.
  \item \textbf{AutoGen} \cite{Wu2023AutoGen} is a real-library implementation of the dialogue-based paradigm, where a group-chat manager orchestrates multi-round multi-agent dialogue.
  \item \textbf{Centralized-LLM (fused)} \cite{Liu2025COHERENT,Kannan2023SMART} is the representative centralized single-brain paradigm, where one central LLM aggregates the whole team's local observations and generates the joint action in one shot, with \emph{fused} denoting a compressed observation input.
  \item \textbf{Centralized-Classical} \cite{Gerkey2004Taxonomy} replaces the central LLM with a classical greedy task-allocation scheduler, representing an LLM-free centralized ceiling on coordination quality.
\end{itemize}

\subsection{Performance Evaluation}

\begin{table}[t]
\centering
\caption{Success rate and per-episode communication (homogeneous 1.5B).}
\label{Tab:MainSort}
\setlength{\tabcolsep}{4pt}
\small
\begin{tabular*}{\columnwidth}{@{\extracolsep{\fill}}lcccc@{}}
\toprule
Method & SR & Data Size (B/ep) & vs.\ ours \\
\midrule
\textbf{Proposed method} & 0.640 & \textbf{43.3} & \textbf{1.0$\times$} \\
AutoGen & 0.700 & 3{,}911 & 90$\times$ \\
RoCo-NL & 0.580 & 3{,}147 & 73$\times$ \\
Centralized-LLM (fused) & 0.667 & 112.7 & 2.6$\times$ \\
Centralized-Classical & 0.733 & 84.3 & 1.9$\times$ \\
\bottomrule
\end{tabular*}
\end{table}

Table~\ref{Tab:MainSort} reports the success rate (SR) and the per-episode communication overhead of the five methods during collaboration with $80$ round in total.
Success rate is the fraction of successful episodes over all evaluation episodes, where an episode succeeds if and only if all task-relevant objects meet the per-object criterion $\varsigma_e$ at termination.
From Table~\ref{Tab:MainSort}, we can see that the considered methods have a similar SR band, while the proposed method has the lowest communication, which is $73$ to $90\times$ lower than the dialogue methods and still $1.9$ to $2.6\times$ lower than the centralized methods. 
This shows that the proposed method matches the baselines' success rate band at the lowest communication overhead.
This is because the dialogue methods must transmit NL every round and the centralized methods must upload every arm's full local observation, whereas each arm of ours completes its semantic understanding locally and uploads only a compact constraint tuple compressible into fixed-width identifiers.

\begin{figure}[t]
\centering
\includegraphics[width=7.5cm]{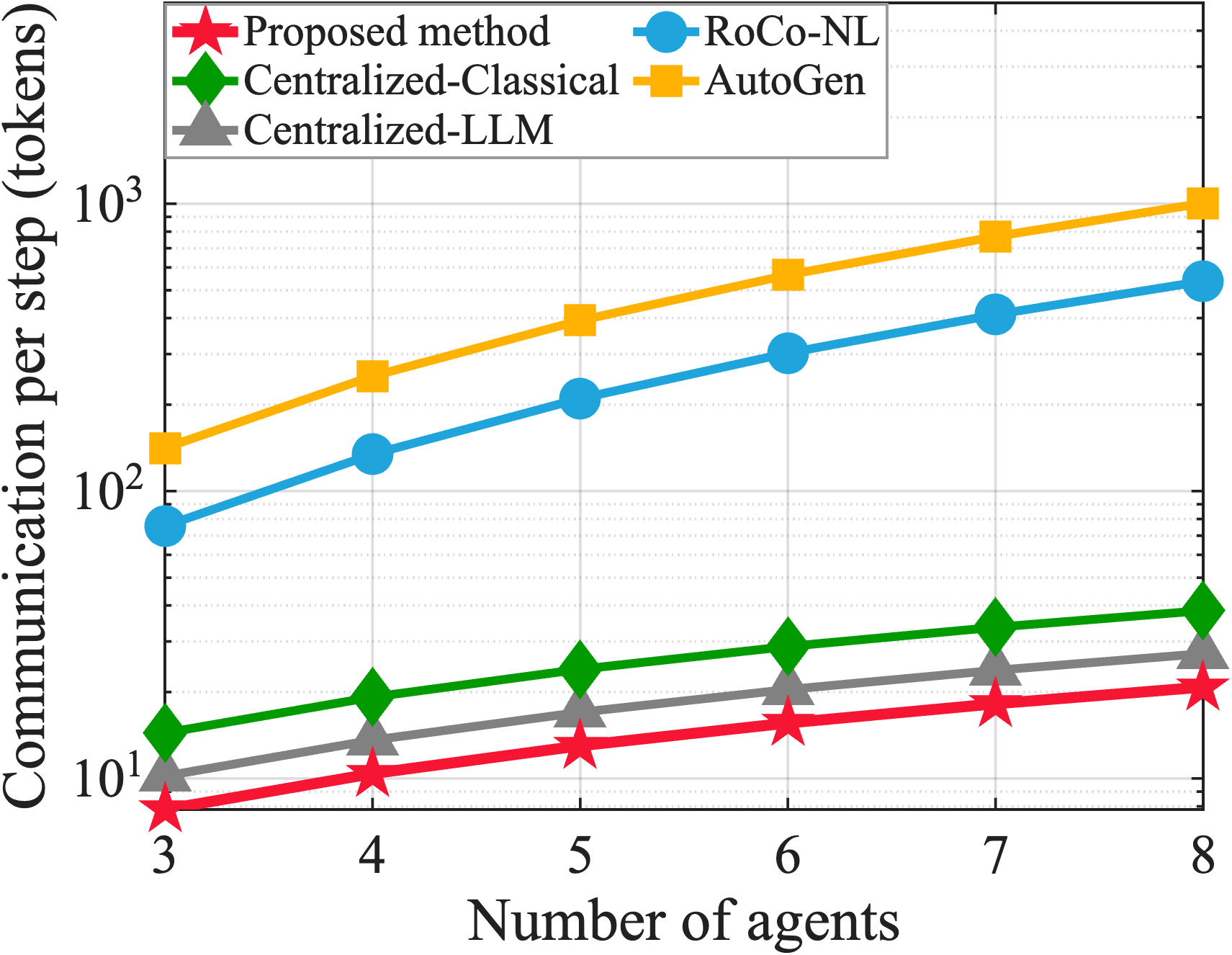}
\caption{Per-step communication vs.\ team size $n$.}
\label{fig_scalability}
\end{figure}

Fig.~\ref{fig_scalability} shows how the per-step communication changes as the team size $n$ varies. 
From Fig.~\ref{fig_scalability}, we can see that RoCo-NL and AutoGen inflate rapidly with $n$, while ours and the two centralized methods grow linearly, and ours stays the lowest throughout. This shows that the communication advantage of ours keeps widening as the team grows and is not reversed by scale. This is because the dialogue payload is proportional to the negotiation rounds times the number of participants and is therefore quadratic, whereas each arm of ours sends only one fixed-length constraint tuple per step so that the total grows linearly with $n$.

\begin{figure}[t]
\centering
\includegraphics[width=7.5cm]{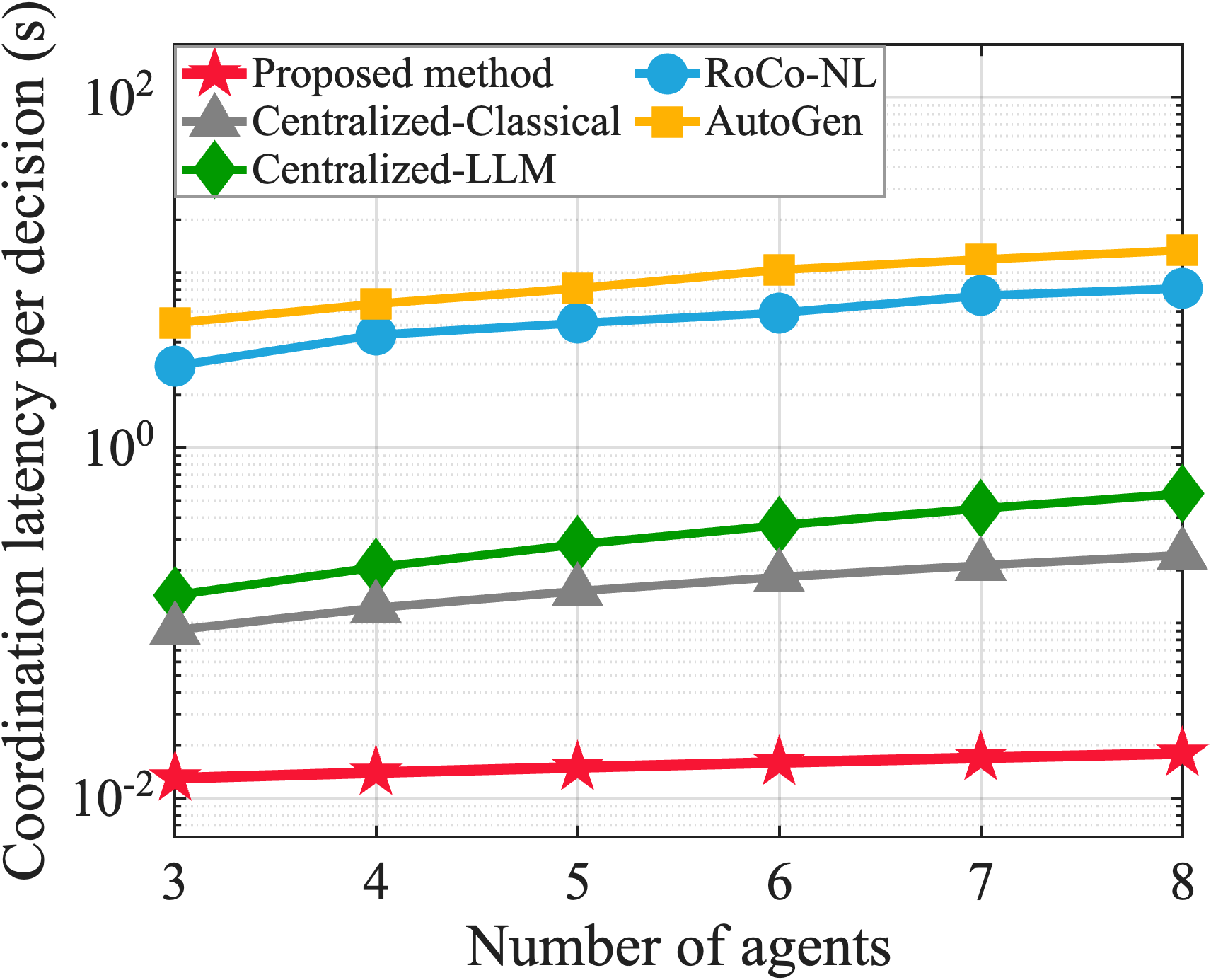}
\caption{End-to-end per-decision coordination latency vs.\ team size $n$.}
\label{fig_latency}
\end{figure}

Fig.~\ref{fig_latency} shows how the end-to-end per-decision latency changes as $n$ varies.
From Fig.~\ref{fig_latency}, we can see that the latency of the proposed method is independent of $n$ and stays a constant $13$ to $18$\,ms and the lowest throughout, while Centralized-Classical rises linearly to $0.24$\,s and Centralized-LLM to $0.55$\,s, and RoCo-NL and AutoGen rise to the order of seconds ($8.1$ and $13.4$\,s).
From Fig.~\ref{fig_latency}, we can also see that at $n=8$ the latency of the proposed method is only about $1/740$ of that of the slowest benchmark, so it alone sustains a real-time-level latency as the team scales.
This is because each arm of ours reads a bounded local context and scores in parallel, whereas the dialogue methods decode autoregressively, round by round, over a context that accumulates with $n$ during LLM inference, and the centralized methods must collect every observation before serial processing.

\begin{figure}[t]
\centering
\includegraphics[width=7.5cm]{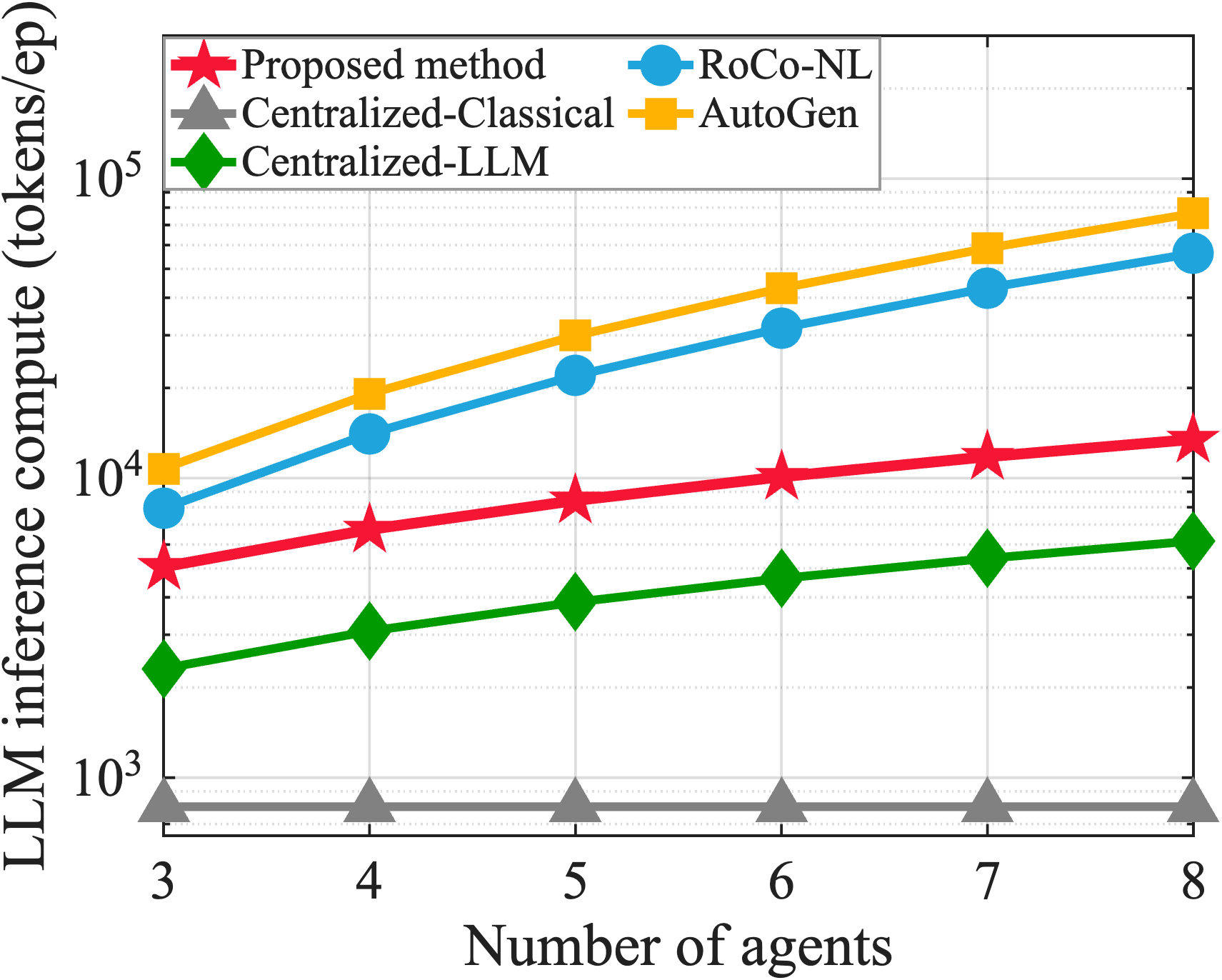}
\caption{Per-episode LLM inference compute vs.\ team size $n$.}
\label{fig_compute}
\end{figure}

Fig.~\ref{fig_compute} shows how the per-episode LLM inference compute (the total of prompt and completion tokens) changes as $n$ varies.
From Fig.~\ref{fig_compute}, we can see that the proposed method grows linearly and RoCo-NL and AutoGen inflate quadratically.
We can also see that at $n=8$, the dialogue methods require $4$ to $6$ times as much computation as the proposed method. 
This shows that the dialogue methods explode not only in communication overhead and latency but also in compute costs. 
This is because the total generated tokens of multi-round dialogue are proportional to the rounds times a growing dialogue context, whereas the total compute of ours is linear in $n$ and parallelizable across arms, so its compute grows while its latency stays flat (cf.\ Fig.~\ref{fig_latency}).

\begin{table}[t]
\centering
\caption{Success rate under increasing team heterogeneity}
\label{Tab:Hetero}
\setlength{\tabcolsep}{4pt}
\small
\begin{tabular*}{\columnwidth}{@{\extracolsep{\fill}}lcccl@{}}
\toprule
Method & Weak heter. & Mid heter. & Strong heter. \\
\midrule
\textbf{Proposed method} & \textbf{0.671} & \textbf{0.656} & \textbf{0.667} \\
RoCo-NL & 0.631 & 0.583 & 0.526 \\
AutoGen & 0.728 & 0.610 & 0.567 \\
Centralized-LLM (fused)$^\dagger$ & 0.667 & 0.667 & 0.667 \\
Centralized-Classical$^\dagger$ & 0.733 & 0.733 & 0.733 \\
\bottomrule
\end{tabular*}
\end{table}

Table~\ref{Tab:Hetero} reports how the SR varies as the team heterogeneity increases from weak to strong.
From Table~\ref{Tab:Hetero}, we can see that the proposed method stays near-flat throughout, while AutoGen degrades monotonically as the heterogeneity grows and RoCo-NL stays lower and drops to its lowest under the strongest heterogeneity.
This shows that, as the team becomes more heterogeneous, the proposed method still maintains coordination, whereas the dialogue methods degrade.
This is because our delegates coordinate correctness to the deterministic DT resolver and depend on no single strong brain, requiring each arm only to express a local intent, whereas the coordination quality of the dialogue methods is directly dragged down by the language ability of the weakest participant. 


\begin{table}[t]
\centering
\caption{Effect of AutoGen dialogue rounds on communication and success rate (Sort, 30 seeds).}
\label{Tab:Rounds}
\setlength{\tabcolsep}{4pt}
\small
\begin{tabular*}{\columnwidth}{@{\extracolsep{\fill}}ccc@{}}
\toprule
Dialogue rounds $R$ & Communication data size\ (msg tok/ep) & SR \\
\midrule
1 & 377 & 0.437 \\
2 & 674 & 0.491 \\
3 & 978 & 0.552 \\
4 & 1{,}282 & 0.621 \\
6 & 1{,}895 & 0.633 \\
\bottomrule
\end{tabular*}
\end{table}

Table~\ref{Tab:Rounds} reports the communication overhead and SR of AutoGen as the number of dialogue rounds $R$ increases from $1$ to $6$. As shown in Table~\ref{Tab:Rounds}, the communication overhead increases by $5$ times, whereas the SR improves by only $44\%$. 
This indicates that dialogue-based methods cannot obtain proportional performance gains by simply increasing the number of negotiation rounds. The reason is that the gains from deeper negotiation quickly saturate and are disproportionate to its communication cost, as the residual coordination bottleneck shifts to physical execution.
Therefore, additional dialogue rounds cannot overcome the execution-level performance ceiling, while the proposed method can achieve a comparable SR with a single communication round.

\subsection{Ablation Study}

\begin{table*}[t]
\centering
\caption{Component ablation by removing one module of LDT-Coord.}
\label{Tab:Ablation}
\setlength{\tabcolsep}{6pt}
\small
\begin{tabular*}{\textwidth}{@{\extracolsep{\fill}}llcccc@{}}
\toprule
Variant & Removed & SR$^\dagger$ & Latency\ ($n{=}3$) & Communication data size (rep./step) & Invalid execution attempts\ (exec./ep) \\
\midrule
full & --- & 0.71 & {0.13}\,s & 0.30 & {5.48} \\
w/o DT & DT constraint resolver & 0.53 & 0.11\,s & 0.30 & {7.12} \\
w/o RL & learned gate $\pi^c$ & 0.70 & 0.13\,s & {1.51}  & 5.48 \\
\bottomrule
\end{tabular*}
\end{table*}

Table~\ref{Tab:Ablation} reports how the four metrics change when one core component is removed, where SR and latency are controlled and near-invariant (see the setup) so that each ablation degrades only its target metric (in bold). From Table~\ref{Tab:Ablation}, we can see that removing the DT constraint resolver raises the invalid execution attempts by about $30\%$ more collision trials per episode, and removing the learned gate raises the reporting needed to reach the same coordination quality by $5$ times reports per step.
This shows that the DT contributes to execution efficiency and the RL gate contributes to communication efficiency, and neither expresses its value through the success rate. 
This is due to the fact that DT vetoes mutual-exclusion conflicts with a one-bit STOP and blocks collisions before execution, while the learned gate reports only at the conflict-prone handoff steps and is communication-efficient. 
\subsection{Convergence and Mechanism: Dual Dynamics of the Learned Gate}

\begin{figure}[t]
\centering
\includegraphics[width=8.5cm]{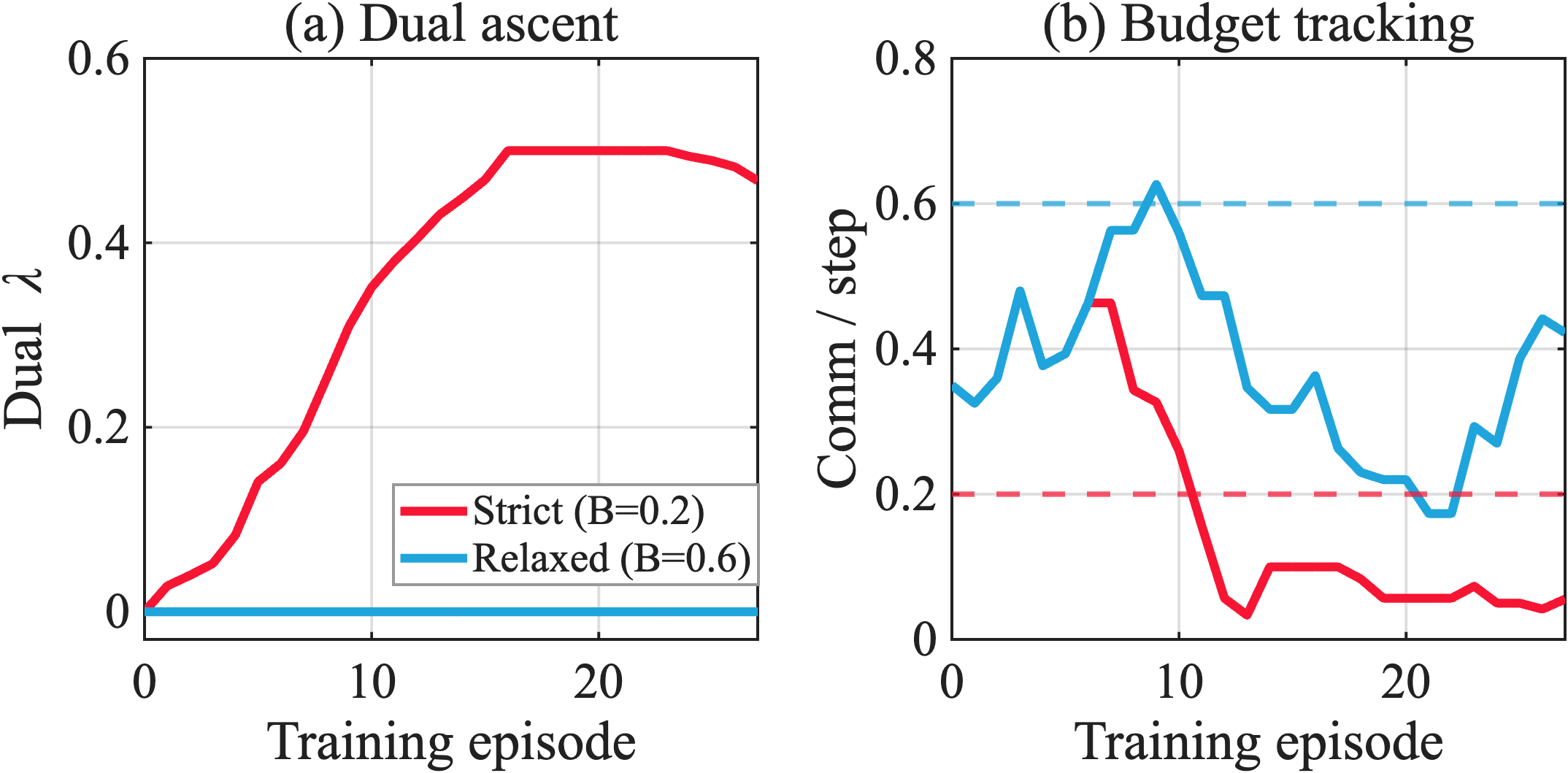}
\caption{Dual dynamics of the learned communication gate $\pi^c$ under tight and loose latency deadlines over training episodes: (a) dual variable $\lambda$, (b) per-step report count, (c) success rate.}
\label{fig_dual}
\end{figure}

Fig.~\ref{fig_dual} shows the dual dynamics of the learned communication gate $\pi^c$ under tight and loose latency deadlines as training proceeds, with the three panels giving the dual variable $\lambda$, the per-step report count, and the success rate.
From Fig.~\ref{fig_dual}, we can see that $\lambda$ rises monotonically with training and drives the per-step report count down into the latency budget, while the SR stays near the execution ceiling throughout.
This is because the PID dual ascent encodes the communication constraint into the reward, so that a tighter constraint yields a larger $\lambda$ and a heavier reporting penalty and forces the policy to report only at the most necessary conflict steps, reaching a low-communication, high-success tradeoff within the feasible region and satisfying the per-step latency deadline.

\section{Conclusion}
This paper presented LDT-Coord, a digital-twin coordination middleware for networks of heterogeneous LLM-driven embodied agents. 
In LDT-Coord, each agent perceives and decides autonomously and reports only a structured action together with typed temporal constraints, while the digital twin runs a single training-free orchestrator that unifies mutual-exclusion, synchronization, and dependency atomic tasks into four orchestration rules and applies them iteratively until convergence.
To further reduce the upload communication overhead, a PPO-Lagrangian-based communication-efficient layer is introduced.
Simulation results show that the proposed coordinate-without-training and communicate-by-learned-optimization framework can enable efficient collaboration among embodied agent teams.
For future work, we will investigate collaborative perception among agent teams to extend each agent’s perceptual boundary.

\appendix
\subsection{Proof of Proposition~\ref{prop:convergence}}
\label{app:convergence}
\begin{proof}
Each rule only removes elements, so the composite operator satisfies the shrinkage property $\Theta(\mathcal{A})\subseteq\mathcal{A}$. Let $\mathcal{A}^{(0)}=\mathcal{N}_t^{\mathrm{act}}$ and $\mathcal{A}^{(k+1)}=\Theta(\mathcal{A}^{(k)})$, so that $\mathcal{A}^{(0)}\supseteq\mathcal{A}^{(1)}\supseteq\cdots$ is a monotonically decreasing chain on the power-set lattice of the finite set $\mathcal{N}_t^{\mathrm{act}}$. Any nontrivial iteration removes at least one element; otherwise, $\mathcal{A}^{(k+1)}=\mathcal{A}^{(k)}$ is already a fixed point. Since the set is finite, the decreasing chain stabilizes after at most $K\le|\mathcal{N}_t^{\mathrm{act}}|$ steps, and the stable point is the greatest fixed point of the shrinkage operator on this finite power-set lattice~\cite{Kung1981Optimistic}. 
This ends the proof.
\end{proof}

\def\baselinestretch{1.0}
\bibliographystyle{IEEEtran}
\bibliography{references}

\end{document}